\def\year{2018}\relax
\documentclass[letterpaper]{article} %DO NOT CHANGE THIS
\usepackage{aaai18}  %Required
\usepackage{times}  %Required
\usepackage{helvet}  %Required
\usepackage{courier}  %Required
\usepackage{url}  %Required
\usepackage{graphicx}  %Required

\usepackage[table]{xcolor}
\usepackage{amsmath}
\usepackage{amsfonts}
\usepackage{multirow}
\usepackage{amssymb}
\usepackage[noend]{algorithmic}
\usepackage{algorithm}

\usepackage{verbatim}

\newcommand{\cor}[1]{
\cellcolor{orange!40}
}

\newcommand{\cg}[1]{
\cellcolor{gray!40}
}

\newtheorem{definition}{Definition}
\newtheorem{example}{Example}

\newtheorem{lemma}{Lemma}
\newtheorem{theorem}{Theorem}
\newtheorem{corollary}{Corollary}
\newtheorem{theoremappendix}{Theorem}
\newenvironment{proof}{\par\noindent{\em Proof.}}{\hfill $\Box$\medskip}

\newcommand{\mc}{\mathcal}
\newcommand{\ie}{i.e., }

\newcommand{\etal}{et al.}

\begin{document}
\title{Asymmetric Action Abstractions \\ for Multi-Unit Control in Adversarial Real-Time Games}
\author{Rubens O. Moraes and Levi H. S. Lelis\\
%	rubens.moraes,levi.lelis@ufv.br \\
	Departamento de Inform\'atica, 
        Universidade Federal de Vi\c{c}osa, Brazil \\
        \{rubens.moraes, levi.lelis\}@ufv.br
}
\maketitle

\begin{abstract}
Action abstractions restrict the number of legal actions available during search in multi-unit real-time adversarial games, thus allowing algorithms to focus their search on a set of promising actions. 
Optimal strategies derived from un-abstracted spaces are guaranteed to be no worse than optimal strategies derived from action-abstracted spaces. In practice, however, due to real-time constraints and the state space size, one is only able to derive good strategies in un-abstracted spaces in small-scale games. 
In this paper we introduce search algorithms that use an action   abstraction scheme we call asymmetric abstraction. Asymmetric abstractions retain the un-abstracted spaces' theoretical advantage over regularly abstracted spaces while still allowing the search algorithms to derive effective strategies, even in large-scale games. Empirical results on combat scenarios that arise in a real-time strategy game show that our search algorithms are able to substantially outperform state-of-the-art approaches. 
\end{abstract}

\section{Introduction}

In real-time strategy (RTS) games the player controls dozens of units to collect resources, build structures, and battle the opponent. 
RTS games are excellent testbeds for Artificial Intelligence methods because they offer fast-paced environments, where 
players act simultaneously, and the number of legal actions 
grows exponentially with the number of units the player controls. 
%a player can take can be very large \cite{OntanonSURCP13}. 
Also, %due to the real-time constraints, 
the time allowed for planning is on the order of milliseconds. 
%We assume two-player deterministic games in which all units are visible to both players.  
%(i.e., the opponent's units are not hidden as it happens in some RTS games). 
In this paper we focus on the combat scenarios that arise in RTS games. 
%control problem of combat units. 
A simplified version of RTS combats in which the units cannot move was shown to be PSPACE-hard in general \cite{FurtakB10}.

A successful family of algorithms for controlling  combat units uses what we call 
action abstractions to reduce the number of legal actions available during the game. 
In RTS games, player actions are represented as a vector of unit moves, where each entry in the vector represents a move for a unit controlled by the player. Action abstractions reduce the number of legal actions a player can perform by reducing the number of legal moves each unit can perform.  

Churchill and Buro \shortcite{ChurchillB13} introduced a method for building action abstractions through 
%what they called 
scripts. A script $\bar{\sigma}$ is a function mapping a game state $s$ and a unit $u$ to a move $m$ for $u$. 
A set of scripts $\mc{P}$ induces an action abstraction by restricting the set of legal moves of all units to moves returned by the scripts in $\mc{P}$. We call an action abstraction created with Churchill and Buro's scheme a uniform abstraction. 

%In theory, players following a strategy derived from an action-abstracted space can be exploited by a player following a strategy derived from an un-abstracted space. This is because the latter have access to actions that the former cannot anticipate for. 
In theory, players searching in un-abstracted spaces are guaranteed to derive optimal strategies that are no worse than the optimal strategies derived from action-abstracted spaces. This is because the former has access to actions that are not available in action-abstracted spaces. 
Despite its theoretical disadvantage, 
uniform abstractions are successful in large-scale combats \cite{ChurchillB13}.  
%with more than 16 combat units \cite{ChurchillB13}. 
This happens because the state space of RTS combats can be very large, and the problem's real-time constraints often allow search algorithms to explore only a small fraction of all legal actions before deciding on which action to perform next---uniform abstractions allow algorithms to focus their search on actions deemed as promising by 
the set of  scripts 
$\mc{P}$. 

%The main contribution of this paper 
%Our main contribution 
In this paper we introduce search algorithms that use what we call 
%is the introduction of 
asymmetric action abstractions (asymmetric abstractions for short) for multi-unit adversarial games. 
In contrast with uniform abstractions that restrict the number of  moves of all units, asymmetric abstractions restrict the number of  moves of only a subset of units. 
We show that asymmetric abstractions 
retain the un-abstracted spaces' theoretical advantage over uniformly abstracted ones while still allowing algorithms to derive effective strategies in practice, even in large-scale games. 
%
%One of 
Another advantage of asymmetric abstractions 
%over other abstraction schemes 
is that they allow the search effort to be distributed unevenly amongst the units. This is important because
%, depending on the game state, 
some units might benefit more from finer strategies (i.e., strategies computed while accounting for a larger set of moves) than others (e.g., in RTS games it is advantageous to provide finer control to units with low hit points so they survive longer). % in the match. 

%be more important than others and one is able to derive stronger strategies by offering a finer control to the more important units. 
%allow one to define spaces that are not restrictive as those defined by uniform abstractions, but still allow algorithms to focus their search on a set of promising  actions. 
%
%As a result of not being as restrictive as uniform abstractions, search algorithms using asymmetric abstractions have the theoretical advantage over players that search in uniformly abstracted spaces as the former has access to actions not available to the latter. Moreover, asymmetric abstractions 
%are able to exploit opponents that follow strategies derived from uniform abstractions. 
%We introduce several strategies for constructing asymmetric abstractions and two algorithms for searching in asymmetrically abstracted spaces. 
%
%Another contribution of this paper is the introduction of two algorithms that search in asymmetrically abstracted spaces. 
The algorithms we introduce for searching in asymmetrically abstracted spaces are based on Portfolio Greedy Search (PGS) \cite{ChurchillB13} and Stratified Strategy Selection (SSS) \cite{lelis2017}, two state-of-the-art approaches. % for uniformly abstracted spaces. 
%Our algorithms are guaranteed to return actions that are at least as good as the actions returned by PGS and SSS in terms of heuristic value. 
Empirical results on RTS combats show that our algorithms 
%searching in asymmetrically abstracted spaces 
are able to substantially outperform PGS and SSS. 

%Although in this paper we use asymmetric abstractions to enhance search-based algorithms, our scheme can potentially be used to enhance other algorithms such as Reinforcement Learning algorithms \cite{Sutton:1998}. In addition to the development of novel algorithms that use our abstraction schemes, we expect future works to investigate other ways of creating asymmetric abstractions. 

\section{Related Work}

%In addition to PGS and SSS, 
Justesen \etal\ \shortcite{JustesenTTR14} 
proposed two variations of UCT \cite{Kocsis:2006} for searching in uniformly abstracted spaces: script-based and cluster-based UCT. Wang et al. \shortcite{WangCLHT16} introduced Portfolio Online Evolution (POE) a local search algorithm also designed for uniformly abstracted spaces. Wang et al. showed that POE is able to outperform Justesen's algorithms, and Lelis \shortcite{lelis2017} showed that PGS and SSS are able to outperform POE. 
%As any algorithm searching in action abstraction spaces, 
Justesen \etal's and Wang \etal's algorithms can also be modified to search in asymmetrically abstracted spaces.
 %we introduce in this paper. 
 We use PGS and SSS in this paper as they are the current state-of-the-art search-based algorithms for RTS combat scenarios \cite{lelis2017}.

%Lelis \shortcite{lelis2017} showed that they are a

Before the invention of action abstractions induced by scripts, state-of-the-art algorithms included search methods for un-abstracted spaces such as  
Monte Carlo \cite{ChungBS05,SailerBL07,Balla2009,Ontanon13} and Alpha-Beta \cite{ChurchillSB12}. Due to the large number of actions available during search, Alpha-Beta and Monte Carlo methods perform well only when controlling a small number of units. Some search algorithms cited are more general than 
the algorithms we consider in this paper, e.g., \cite{Ontanon13,OntanonB15}. This is because such algorithms can be used to control a playing agent throughout a complete RTS game. By contrast, the algorithms we consider in this paper are specialized for combat scenarios. 

Another line of research uses learning to control combat units in RTS games. 
Search algorithms need an efficient forward model of the game to plan. By contrast, learning approaches do not necessarily require such a model. 
% to be available as they  learn from their past experiences. 
Examples of learning approaches to unit control  
include the work by Usunier \etal\ \shortcite{UsunierSLC16} and Liu \etal\  \shortcite{LiuLB16}. Likely due to the use of an efficient forward model, search algorithms tend to scale more easily to large-scale combat scenarios than learning-based methods. While the former can effectively handle battles with more than 100 units, the latter are usually tested on battles with no more than 50 units. 
%For a review of RTS research, see the work by Onta{\~{n}}{\'{o}}n \etal \shortcite{OntanonSURCP13}.

%State-space and action abstractions have also been applied to imperfect-information extensive-form games, all with applications in computer Poker \cite{Gilpin:2007,HawkinHS11,Johanson:2013,Bard:2014}. Although simultaneous-moves games such as the ones we deal with in this paper can also be modeled as an imperfect-information extensive-form game \cite{BosanskyLLCW16}, the abstraction schemes used to develop Poker playing agents differ considerably from the approaches we consider and introduce in this paper. The state space of RTS games are often far too large to employ the approach used in Poker playing agents, where a strategy is computed offline and stored to be consulted during the game. We derive our strategy during the game, and in contrast with computer Poker, RTS games operate under harsh real-time constraints. 

\section{Preliminaries}

Combat scenarios that arise in RTS games, which we also call \textbf{matches}, can be described as finite zero-sum two-player games with simultaneous and durative moves. We assume matches with deterministic actions in which all units are visible to both players.  
Matches can be defined by a tuple $(\mc{N}, \mc{S}, s_{init}, \mc{A}, \mc{R}, \mc{T})$, where, 

\begin{itemize}
\item $\mc{N} = \{i, -i\}$ is the set of \textbf{players} ($i$ is the player we control and $-i$ is our opponent).
\item  $\mc{S} = \mc{D} \cup \mc{F}$ is the set of \textbf{states}, where $\mc{D}$ denotes the set of \textbf{non-terminal states} and $\mc{F}$ the set of \textbf{terminal states}. Every state $s \in \mc{S}$ defines a grid map containing a joint set of \textbf{units} $\mc{U} = \mc{U}_i \cup \mc{U}_{-i}$, for players $i$ and $-i$. Every unit $u \in \mc{U}$ has properties such as $u$'s $x$ and $y$ coordinates on the map,  
attack range ($r(u)$), attack damage ($d(u)$), %(which depends on the unit $u$ and move $m$ being performed),  
hit points ($hp(u)$), and weapon cool-down time, \ie the time the unit has to wait before repeating an attack action ($cd(u)$). 
%
%\item 
$s_{init} \in \mc{D}$ is the start state and defines the initial position of the units $\mc{U}$ on the map.
\item $\mc{A} = \mc{A}_i \times \mc{A}_{-i}$ is the set of \textbf{joint actions}. % for both players. 
$\mc{A}_i(s)$ is the set of legal \textbf{actions} player $i$ can perform at state $s$. 
Each action $a \in \mc{A}_{i}(s)$ is denoted by a vector of $n$ \textbf{unit moves} $(m_1, \cdots, m_n)$, where $m_k \in a$ is the move of the $k$-th \textbf{ready unit} of player $i$. A unit $u$ is not ready at $s$ if $u$ is busy performing a move. We denote the set of ready units of players $i$ and $-i$ as $\mc{U}^r_i$ and $\mc{U}^r_{-i}$. For $k \in \mathbb{N}^+$ we write $a[k]$ to denote the move of the $k$-th ready unit. Also, for unit $u$, we write $a[u]$ to denote the move of $u$ in $a$.

\item We denote the set of unit moves as $\mc{M}$, which includes moving up ($U$), left ($L$), right ($R$) and down ($D$), waiting ($W$), and attacking an enemy unit. The effect of moves $U, L, R, D$ is to change the unit $x$ and $y$ coordinates on the map; the effect of an attack move is the reduction of the target unit's $hp$ value by the $d$-value of the unit performing the attack. We write $\mc{M}(s, u)$ to denote the set of legal moves of unit $u$ at $s$. 
%
%\item We denote the set of ready units of players $i$ and $-i$ as $\mc{U}^r_i$ and $\mc{U}^r_{-i}$. For $k \in \mathbb{N}^+$ we write $a[k]$ to denote the move of the $k$-th ready unit. Also, for unit $u$, we write $a[u]$ to denote the move of $u$ in $a$.
%
%\item $\mc{R} : \mc{F} \times \mc{N} \rightarrow \mathbb{R}$ is a \textbf{utility function} 
%with $\mc{R}(s, i) = -\mc{R}(s, -i)$, for any $s \in \mc{F}$.
\item $\mc{R}_i : \mc{F} \rightarrow \mathbb{R}$ is a \textbf{utility function} 
with $\mc{R}_i(s) = -\mc{R}_{-i}(s)$, for any $s \in \mc{F}$. We use the LTD2 formula introduced by Kovarsky and Buro \shortcite{Kovarsky2005} as utility function. %LTD2 accounts for both players' units $hp$ and $dpf$. 
LTD2 evaluates a state $s$ with $\mc{U} = \mc{U}_i \cup \mc{U}_{-i}$ as follows. 
\begin{equation*}
%LTD2(\mc{U}_i, \mc{U}_{-i}) = 
\sum_{u \in \mc{U}_i} \sqrt{hp(u)} \cdot dpf(u) - \sum_{u \in \mc{U}_{-i}} \sqrt{hp(u)} \cdot dpf(u) \,.
\end{equation*}
Here, $dpf$ is the amount of damage $u$ can cause per frame of the game and is defined as $dpf(u) = d(u)/(cd(u) + 1)$ (we use $cd(u) + 1$ to ensure a valid operation if $cd = 0$).

\item The \textbf{transition function} $\mc{T} : \mc{S} \times \mc{A}_i \times \mc{A}_{-i} \rightarrow \mc{S}$ determines the sucessor state for a state $s$ and the set of joint actions taken at $s$.
\end{itemize}

A \textbf{decision point} of player $i$ is a state $s$ in which $i$ has at least one ready unit. In the framework we consider in this paper, a search algorithm is invoked at every decision point to decide on the player's next action. %The algorithm is allowed 40 milliseconds to search and return an action with one move for each ready unit. % at $s$. 

The \textbf{game tree} of a match is a tree rooted at $s_{init}$ whose nodes represent states in $\mc{S}$ and every edge represents a joint action in $\mc{A}$. %; the tree is rooted at $s_{init}$.  
%The search tree is rooted at $s_{init}$ and all reachable nodes from the root are decision points. 
For states $s_k, s_j \in \mc{S}$, there exists an outgoing edge from $s_k$ to $s_j$ if and only if there exists $a_i \in \mc{A}_i$ and $a_{-i} \in \mc{A}_{-i}$ such that $\mc{T}(s_k, a_i, a_{-i}) = s_j$. Nodes representing states in $\mc{F}$ are leaf nodes. We assume all matches to be finite, i.e., that the tree is bounded. % with the match finishing after a finite number of decision points. 
We denote as $\Psi$ the \textbf{evaluation function} used by search algorithms while traversing the game tree. $\Psi$ receives as input a state $s$ and returns an estimate of the end-game value of $s$ for player $i$. %We write the $\Psi$-value of an action $a$ to mean the $\Psi$-value of the successor of a state $s$ obtained after applying $a$ to $s$.  

%\begin{definition}
%The game tree (GT) of a match is a tree in which each node represents a state in $\mc{S}$ and every edge represents a joint action in $\mc{A}$. The GT is rooted at $s_{init}$ and 
%%all reachable nodes from the root are decision points. 
%for states $s_k, s_j \in \mc{S}$, there exists an outgoing edge from $s_k$ to $s_j$ if and only if there exist $a_i \in \mc{A}_i$ and $a_{-i} \in \mc{A}_{-i}$ such that $\mc{T}(s_k, a_i, a_{-i}) = s_j$. Nodes representing states in $\mc{F}$ are leaf nodes. 
%\end{definition}

%
A \textbf{player strategy} is a function $\sigma_i : \mc{S} \times \mc{A}_i \rightarrow [0, 1]$ for player $i$, which maps a state $s$ and an action $a$ to a probability value, indicating the chance of taking action $a$ at $s$.   
%We denote as $\Sigma_i$ and $\Sigma_{-i}$ as the set of all player strategies for players $i$ and $-i$. 
A strategy profile $\sigma = (\sigma_i, \sigma_{-i})$ defines the strategy of both players. 
%
%We denote as $\mc{R}(\sigma, i)$ the utility of player $i$ if the player strategy profile $\sigma$ is employed. 
%We denote as $\pi_s^\sigma(z)$ as the probability of reaching terminal state $z \in \mc{Z}$ while following $\sigma$ from state $s$. 
%
The optimal \textbf{value of the game} rooted $s$ for player $i$ 
is denoted as $V_i(s)$ and can be computed by finding a Nash Equilibrium profile. 
%Backward induction methods \cite{ross_1971} can be used to find %Nash equilibrium strategies
%optimal profiles.  
% by searching in the game tree. 
%However, 
Due to the problem's size and real-time constraints, %it is impractical to derive Nash equilibrium strategies 
it is impractical to find optimal profiles 
for most RTS combats. % of the matches that arise in RTS games. %Thus, instead of searching in the GT, 
State-of-the-art approaches use abstractions to reduce the game tree size and then derive player strategies from the abstracted trees.

\section{Uniform Action Abstractions}

We define a \textbf{uniform abstraction} for player $i$ as a function mapping the set of legal actions $\mc{A}_i$ to a subset $\mc{A}'_i$ of $\mc{A}_i$. In RTS games, action abstractions are constructed from a collection of scripts. A \textbf{script} $\bar{\sigma}$ is a function mapping a state $s$ and a unit $u$ in $s$ to a legal move $m$ for $u$. 
A script $\bar{\sigma}$ can be used to define a player strategy $\sigma_i$ by applying $\bar{\sigma}$ to every  unit in the state. We write $\bar{\sigma}$ instead of $\bar{\sigma}(s, u)$ whenever the state and the unit are clear from the context.  

Let the \textbf{action-abstracted legal moves} of $u$ at state $s$ be the moves for $u$ that is returned by a script in $\mc{P}$, defined as, 
\begin{equation*}  
\mc{M}(s, u, \mc{P}) = \{\bar{\sigma}(s, u) | \bar{\sigma} \in \mc{P}\} \,.
\end{equation*}
\begin{definition} %[Uniform abstraction]
A uniform abstraction $\Phi$ is a function receiving as input % $(s, i, \mc{P})$, where  
 a state $s$, 
%$\nabla$ is a match, 
a player $i$, 
%in $\nabla$, 
and a set of scripts $\mc{P}$. 
% is 
%induced by a set of scripts $\mc{P}$ for a player $i$ in state $s$. 
%defined for player $i$, 
%for a set of units $\mc{U}_i$ in 
%state $s$, and a set of scripts $\mc{P}$. induces a uniform abstraction by defining 
%For any state $s$ in $\nabla$, 
$\Phi$ returns a subset of $\mc{A}_i(s)$ denoted $\mc{A}'_i(s)$. $\mc{A}'_i(s)$ is defined by the Cartesian product of moves in $\mc{M}(s, u, \mc{P})$ for all $u$ in $\mc{U}^r_i$, where $\mc{U}^r_i$ is the set of ready units of $i$ in $s$. 
%We use $\mc{A}'_i$ as a shorthand for $\Phi(s, i, \mc{P})$. 
%We write $\Phi(i, \mc{P})$ instead of $\Phi(s, i, \mc{P})$ whenever $s$ is clear from the context. 
%returned by scripts in $\mc{P}$. 
\end{definition}

%For any unit $u \in \mc{U}_i$, $\mc{A}'_i$ does not contain an action $a$ if $a$ contains a move $m$ for $u$ that is not in 
%the set of 
Algorithms using a uniform abstraction search in a game tree for which player $i$'s legal actions are limited to $\mc{A}'_i(s)$ for all $s$. This way, algorithms focus their search on actions deemed as promising by the scripts in $\mc{P}$, as the actions in $\mc{A}'_i(s)$ are composed of moves returned by the scripts in $\mc{P}$. 

NOKAV and Kiter are scripts commonly used for inducing uniform abstractions \cite{ChurchillB13}. NOKAV assigns a move to $u$ so that $u$ does not cause more damage than that required to set an enemy's unit $hp$ to zero. Kiter allows $u$ to attack and then move away from its target. %Cluster computes the centroid of units as the average $x$ and $y$ of all units in $\mc{U}_i$. Then, Cluster assigns a move to a unit so it moves in the direction of the centroid. % units already at the centroid will be assigned the wait move. %See the work by Churchill and Buro \shortcite{ChurchillB13} for more details on NOKAV, Kiter and Cluster. 

%Two scripts commonly used for inducing uniform abstractions are known as NOKAV and Kiter \cite{ChurchillB13}. NOKAV assigns a move to a unit $u$ so that $u$ does not cause more damage than the required to reduce an enemy's unit $hp$ to zero. Kiter allows a unit to attack and then move away from the target. See the work by Churchill and Buro \shortcite{ChurchillB13} for more details on NOKAV and Kiter.

%We use two versions of $\Theta$ in Algorithm \ref{alg:haas}, one that receives one action vector $a$, a state $s$, and the player who is to perform $a$ in $s$ as input (see lines \ref{haas:eval1} and \ref{haas:eval2}) and another that receives two action vectors $a_1$ and $a_2$, a state $s$, and the player $i$ who is to perform $a_1$ at $s$ as input (see line \ref{haas:comp}). When receiving one action vector $a$ as input, $\Theta$ estimates the end-game value of $s$ after player $i$ performs action $a$ in $s$. 
%When receiving two actions $a_1$ and $a_2$, $\Theta$ estimates the end-game value after one player has taken action $a_1$ and its opponent action $a_2$. 

\begin{algorithm}[t]
%\footnotesize
%\scriptsize
\caption{Portfolio Greedy Search}
\label{alg:pgs}
\begin{algorithmic}[1]
\REQUIRE state $s$, available units $\mc{U}^r_i = \{u^i_1, \cdots, u^i_{n_i}\}$ and $\mc{U}^r_{-i} = \{u^{-i}_1, \cdots, u^{-i}_{n_{-i}}\}$ in $s$, unit strategies $\mc{P}$, time limit $t$, and evaluation function $\Psi$.
\ENSURE action $a$ for player $i$'s units.
%\STATE $\bar{\sigma}_i \gets \underset{\bar{\sigma} \in \Sigma}{\argmax} \, \Theta((\bar{\sigma}(u^i_1), \cdots, \bar{\sigma}(u^i_{n_i})), s, i)$ \label{haas:eval1}
%\STATE $\bar{\sigma}_{-i} \gets \underset{\bar{\sigma} \in \Sigma}{\argmax} \, \Theta((\bar{\sigma}(u^{-i}_1), \cdots, \bar{\sigma}(u^{-i}_{n_{-i}})), s, -i)$ \label{haas:eval2}
\STATE $\bar{\sigma}_i \gets $ choose a script from $\mc{P}$ \textit{//see text for details} \label{haas:eval1}
\STATE $\bar{\sigma}_{-i} \gets $ choose a script from $\mc{P}$ \textit{//see text for details} \label{haas:eval2}
\STATE $a_i \gets \{\bar{\sigma}_i(u^i_1), \cdots, \bar{\sigma}_i(u^i_{n_i})\}$ %// vector with $n_i$ elements
\STATE $a_{-i} \gets \{\bar{\sigma}_{-i}(u^{-i}_{1}), \cdots, \bar{\sigma}_{-i}(u^{-i}_{n_{-i}})\}$ \label{haas:init_port} %// vector with $n_{-i}$ elements 
\WHILE{time elapsed is not larger than $t$}
\FOR{$k \gets 1$ to $|\mc{U}^r_i|$} \label{pgs:for_k}
%	\IF{$u^i_k \in \mc{U}_i^{\Delta}$} 
%		\STATE $M \gets \mc{M}(s, u^i_k, \Sigma)$ \textit{//reduced number of moves} \label{haas:reduced}
%	\ELSE
%		\STATE $M \gets \mc{M}(s, u^i_k)$ \textit{//all legal moves} \label{haas:all}
%	\ENDIF
%		\FOR{each $m \in \mc{M}(s, u^i_k, \mc{P})$}
		\FOR{each $\bar{\sigma} \in \mc{P}$}
			\STATE $a'_i \gets a_i$; $a'_i[k] \gets \bar{\sigma}(s, u^i_k)$ 
	%		\STATE 
			\IF{$\Psi(\mc{T}(s, a'_i, a_{-i})) > \Psi(\mc{T}(s, a_i, a_{-i}))$} \label{haas:comp}
				\STATE $a_i \gets a'_{i}$ \label{pgs:change}
			\ENDIF
		\ENDFOR
		\IF{time elapsed is larger than $t$} 
			\STATE \textbf{return} $a_i$ \label{haas:return1}
		\ENDIF
\ENDFOR
\ENDWHILE
\STATE \textbf{return} $a_i$ \label{haas:return2}
\end{algorithmic}
\end{algorithm}

\subsection{Searching in Uniformly Abstracted Spaces}

Churchill and Buro \shortcite{ChurchillB13} introduced PGS, a method for searching in uniformly abstracted spaces. Algorithm \ref{alg:pgs} presents PGS, which receives as input a state $s$, 
player $i$'s and $-i$'s set of ready units for $s$ ($\mc{U}^r_i$ and $\mc{U}^r_{-i}$), 
a set of scripts $\mc{P}$, a time limit $t$, and an evaluation function $\Psi$. % that receives as input a state and returns an estimate of the end-game value of that state. 
PGS returns an action $a$ for player $i$ to be executed in $s$. 
PGS selects the script $\bar{\sigma}_i$ (resp. $\bar{\sigma}_{-i}$) from $\mc{P}$ (lines \ref{haas:eval1} and \ref{haas:eval2}) that yields the largest $\Psi$-value assuming player $i$ executes an action composed of moves 
%the move of all available units of player $i$ is 
computed with $\bar{\sigma}_i$ for all units in $\mc{U}_i^r$ (resp. $\mc{U}_{-i}^r$), assuming $-i$ (resp. $i$) executes an action selected by the NOKAV script. 
%a single strategy is assigned to all available units of a given player (see lines \ref{haas:eval1} and \ref{haas:eval2}). %; this step is done for both $i$ and $-i$, yielding $\bar{\sigma}_i$ and $\bar{\sigma}_{-i}$. 
Action vectors $a_i$ and $a_{-i}$ are initialized with the moves computed from $\bar{\sigma}_i$ and $\bar{\sigma}_{-i}$. %$b_i$ and $b_{-i}$ have sizes $n_i$ and $n_{-i}$, which denote the number of units for $i$ and for $-i$, respectively.

Once $a_i$ and $a_{-i}$ have been initialized, PGS 
iterates through all units $u_k^i$ in $\mc{U}_i^r$ and tries to greedily  
%performs a hill-climbing search to 
improve the move assigned to $u_k^i$ in $a_i$, denoted by $a_i[k]$. 
%while the algorithm has not reached the time limit 
%(lines \ref{pgs:while_start}--\ref{pgs:while_end}). 
%%In PGS's search procedure, 
%For each available unit HAAS checks whether it 
%If $u_k^i$
%is in $\mc{U}_i^{\Delta}$, then 
Since PGS only assigns moves to units given by scripts in $\mc{P}$, it considers only actions in the space induced by a uniform abstraction. % i.e., those in $\mc{M}(s, u_k, \mc{P})$. 
%By contrast, if $u_k^i$ is not in $\mc{U}_i^{\Delta}$, HAAS will consider all legal moves $\mc{M}(s, u^i_k)$ (see line \ref{haas:all}). 
PGS evaluates $a_i$ 
for each possible move $\bar{\sigma}(s, u^i_k)$ for unit $u^i_k$. 
%while replacing $a_i[k]$ by each of the possible moves $\bar{\sigma}(s, u^i_k)$ for $u^i_k$. 
PGS keeps in $a_i$ the action found during search with the largest $\Psi$-value. This process is repeated until PGS reaches time limit $t$. PGS then returns $a_i$. % (see lines \ref{haas:return1} and \ref{haas:return2}).

The action $a_{-i}$ does not change after its initialization (see line \ref{haas:init_port}). Although in PGS's original formulation one 
%
%PGS 
alternates between improving player $i$'s and player ${-i}$'s actions \cite{ChurchillB13}, Churchill and Buro suggested to keep player ${-i}$'s action fixed after initialization as that leads to better results in practice. 

Lelis \shortcite{lelis2017} introduced Stratified Strategy Selection (SSS), a hill-climbing algorithm for uniformly abstracted spaces similar to PGS. The main difference between PGS and SSS is that the latter searches in the space 
%of script assignments 
induced by a partition of units called a type system. 
%SSS uses a partition of units, which we call a \textbf{type system}, and 
SSS assigns moves returned by the same script to units of the same type. For example, all units with low $hp$-value (type) move away from the battle %so that they can survive longer 
(strategy encoded in a script). In terms of pseudocode, SSS initializes $\bar{\sigma}_i$ and $\bar{\sigma}_{-i}$ with the NOKAV script (lines \ref{haas:eval1} and \ref{haas:eval2}). Instead of iterating through all units as PGS does, SSS iterates through all types $q$ of units in line \ref{pgs:for_k} of Algorithm \ref{alg:pgs} and assigns the move provided by $\bar{\sigma}$ to all units of type $q$ before evaluating the resulting state with $\Psi$. 
SSS uses a meta-reasoning method to select the type system to be used. We call SSS what Lelis \shortcite{lelis2017} called SSS+.

%For details on PGS and SSS we refer the reader to their original papers \cite{ChurchillB13,lelis2017}. 

%Despite the empirical success of PGS and SSS, in theory, a player following a strategy derived from a uniformly abstracted space can be exploited by players that account for all legal actions, as PGS and SSS will not be able to anticipate for some of the actions used by their opponents. In practice this is only observed in small-scale combat scenarios---Churchill and Buro \shortcite{ChurchillB13} showed that ABCD outperforms PGS in combat scenarios with 16 units. This is because the number of legal actions increases exponentially with the number of units, and 
%without the guidance of the scripts for pre-selecting a subset of promising actions, one is unlikely to encounter good actions within the time limit. 

\section{Asymmetric Action Abstractions}

Uniform abstractions are restrictive in the sense that all units have their legal moves reduced to those specified by scripts. In this section we introduce an abstraction scheme we call asymmetric  abstraction that is not as restrictive as uniform abstractions but still uses the guidance of the scripts for selecting a subset of promising actions.  
The key idea behind asymmetric abstractions is to   
 reduce the number of legal moves of only a subset of the units  controlled by player $i$; the sets of legal moves of the other units remain unchanged. We call  the subset of units that do not have their set of legal moves reduced the \textbf{unrestricted units}; the complement of the unrestricted units are referred as the \textbf{restricted units}. %In contrast with uniform abstractions that are induced by a set of scripts $\mc{P}$, asymmetric abstractions are induced by a set of scripts $\mc{P}$ and a set of unrestricted units. 

\begin{definition}
An asymmetric abstraction $\Omega$ is a function receiving as input %$(s, i, \mc{U}'_i, \mc{P})$, where  
 a state $s$, 
%$\nabla$ is a match, 
a player $i$, 
a set of unrestricted units $\mc{U}'_i \subseteq \mc{U}^r_i$, 
%which is a subset of units of $i$ in $s$,  
and a set of scripts $\mc{P}$. 
$\Omega$ returns a subset of actions of $\mc{A}_i(s)$, denoted $\mc{A}''_i(s)$, defined by the Cartesian product of the moves in $\mc{M}(s, u, \mc{P})$ for all $u$ in $\mc{U}^r_i \setminus \mc{U}'_i$ and of moves $\mc{M}(s, u')$ for all $u'$ in $\mc{U}'_i$. 
%We write $\Omega(i, \mc{P})$ instead of $\Omega(s, i, \mc{U}'_i, \mc{P})$ whenever $s$ and $\mc{U}'_i$ are clear from the context. 
%We use $\mc{A}''_i$ as a shorthand for $\Omega(s, i, \mc{U}'_i, \mc{P})$. 
\end{definition} 

Algorithms using an asymmetric abstraction $\Omega$ search in a game tree for which player $i$'s legal actions are limited to $\mc{A}''_i(s)$ for all $s$. 
%Algorithms using an asymmetric abstraction $\Omega$ consider only actions in $\mc{A}''(s)$ for all $s$ visited during search. 
 %
%Note that 
If the set of unrestricted units is equal to the set of units controlled by the player, then 
the asymmetric abstraction is equivalent to the un-abstracted space, and if the set of unrestricted units is empty, the asymmetric abstraction is equivalent to the uniform abstraction induced by the same set of scripts. Asymmetric abstractions allow us to explore the action abstractions in the large spectrum of possibilities between the uniformly abstracted and un-abstracted spaces. 

%Depending on the game state, 
%some units might be more important than others (e.g., units with low hit points trying to survive), and asymmetric abstractions allow one to derive finer strategies to these units by accounting for a larger set of moves for them. By contrast, uniform abstractions always divide the search effort equally amongst all units.

The following theorem shows that an optimal strategy derived from the space induced by an asymmetric abstraction is at least as good as the optimal strategy derived from the space induced by a uniform abstraction as long as both abstractions are defined by the same set of scripts. 

\begin{theorem}
Let $\Phi$ be a uniform abstraction and $\Omega$ be an asymmetric abstraction, both defined with the same set of scripts $\mc{P}$. For a finite match with start state $s$, let $V_i^{\Phi}(s)$ be the optimal value of the game computed by considering the space induced by $\Phi$; define $V_i^{\Omega}(s)$ analogously. We have that 
$V_i^{\Omega}(s) \ge V_i^{\Phi}(s)$. 
\label{theorem:abs}
\end{theorem}

The proof for Theorem \ref{theorem:abs} (provided in the Appendix) hinges on the fact that a player searching with $\Omega$ has access to more actions than a player searching with $\Phi$. This guarantee can also be achieved by enlarging the set $\mc{P}$ used to induce $\Phi$. The problem of enlarging $\mc{P}$ is that new scripts might not be readily available as they need to be either handcrafted or learned. By contrast, one can easily create a wide range of asymmetric abstractions by modifying the set of unrestricted units. Also, depending on the combat scenario, 
some units might be more important than others 
%(e.g., units with low $hp$-values), 
and asymmetric abstractions allow one to assign finer strategies to these units. % which might allow them to survive longer. 
Similarly to what human players do, asymmetric abstractions allow algorithms to focus on a subset of units at a given time of the match. 
This is achieved by considering all legal moves of the unrestricted units during search. 

\subsection{Searching with Asymmetric Abstractions}
\label{sec:search} 

We introduce Greedy Alpha-Beta Search (GAB) and Stratified Alpha-Beta Search (SAB), two algorithms for searching in asymmetrically abstracted spaces. GAB and SAB hinge on a property of PGS and SSS that has hitherto been overlooked. Namely, both PGS and SSS may come to an early termination if they encounter a local maximum. PGS and SSS reach a local maximum when they complete all iterations of the outer for loop in Algorithm \ref{alg:pgs} (line \ref{pgs:for_k}) without altering $a_i$ (line \ref{pgs:change}). Once a local maximum is reached, PGS and SSS are unable to further improve the move assignments, even if the time limit $t$ was not reached.  
%Our goal with them is to showcase the potential of 
%asymmetric abstractions. 

GAB and SAB take advantage of PGS's and SSS's early termination by operating in two steps.  
%
%The algorithms we introduce, Greedy Alpha-Beta Search (GAB) and Stratified Alpha-Beta Search (SAB), explore the asymmetric abstraction space in two steps. 
In the first step GAB and SAB search for an action in the uniformly abstracted space with PGS and SSS, respectively. The first step finishes either when (i) the time limit is reached or (ii) a local maximum is encountered. In the second step, which is run only if the first step finishes by encountering a local maximum, GAB and SAB fix the moves of all restricted units according to the moves found in the first step, and search in the asymmetrically abstracted space for moves for all 
unrestricted units. If the first step finishes by reaching the time limit, GAB and SAB return the action determined in the first step. GAB and SAB behave exactly like PGS and SSS in decision points in which the first step uses all time allowed for planning. We explain GAB and SAB in more detail below.

We also implemented a variant of PGS for searching in asymmetric spaces that is simpler than the algorithms we present in this section. In this PGS variant, during the hill-climbing search, for a given state $s$, instead of limiting the number of legal moves of all units $u$ to $\mc{M}(s, u, \mc{P})$, as PGS does, we consider all legal moves $\mc{M}(s, u)$ for unrestricted units, and the moves $\mc{M}(s, u, \mc{P})$ for restricted units. We call this PGS variant Greedy Asymmetric Search (GAS). 
%We do not present the empirical results of this PGS variant because it did not perform as well as GAB and SAB. %the algorithms we explain next. 

\subsubsection{Greedy Alpha-Beta Search (GAB)}

In its first step GAB uses PGS to search in a uniformly abstracted space induced by $\mc{P}$ for deriving an action $a$ that is used to fix the moves of the restricted units during the second search. % to be executed at $s$; 
%PGS runs until it reaches a local maximum or until reaching the time limit. If the former happens, 
%, \ie until it completes a full iteration over all units without modifying the unit-move assignment. 
%Then, 
%PGS returns an 
%$a$ is used in GAB's second step as we explain below.  %GAB returns $a$ as the action to be executed otherwise. 
%
In its second step, GAB uses a variant of Alpha-Beta that accounts for durative moves \cite{ChurchillSB12} (ABCD). Although we use ABCD, one could also use other algorithms such as UCTCD \cite{ChurchillB13}. % for the algorithm's second step. 
ABCD is used to search in a tree we call \textbf{Move-Fixed Tree} ($\mathit{MFT}$). The following example illustrates how the $\mathit{MFT}$ is defined; $\mathit{MFT}$'s definition follows the example. 

\begin{example}
Let $\mc{U}^r_i = \{u_1, u_2, u_3\}$ be $i$'s ready units in $s$, % (all units are ready),  
%Consider an asymmetric action abstraction $A$ induced by 
$\mc{P} = \{\bar{\sigma}_1, \bar{\sigma}_2\}$ be a set of scripts, and $\{u_1, u_3\}$ be the unrestricted units. 
 %a subset of $\mc{U}^a_i$. 
 Also, let $a = (W, L, R)$ be the action returned by PGS while searching in the uniformly abstracted space induced by $\mc{P}$. % during GAB's first step. 
GAB's second step searches in the $\mathit{MFT}$. 

The $\mathit{MFT}$ is rooted at $s$, and the set of abstracted legal actions in $s$ is obtained by fixing $a[u_2] = L$ and considering all legal moves of $u_1$ and $u_3$. That is, if $\mc{M}(s, u_1) = \{W, U\}$ and  $\mc{M}(s, u_3) = \{R, D\}$, then the set of abstracted legal actions in $s$ is: $\{(W, L, R), (W, L, D), (U, L, R), (U, L, D)\}$.
%\begin{equation*}
%\mc{A}_i(s) = \{(W, L, R), (W, L, D), (U, L, R), (U, L, D)\} \,. 
%\end{equation*}s
For all descendants states $s'$ of $s$ in the $\mathit{MFT}$, if 
%we have that 
$\mc{M}(s', u_1) = \{W, U\}$ and  $\mc{M}(s', u_3) = \{R, D\}$, then the set of abstracted legal actions in $s'$ is:
\begin{align*}
%\mc{A}_i(s') =
& \{(W, \bar{\sigma}_1(s', u_2), R),  (W, \bar{\sigma}_1(s', u_2), D), \\
& (U, \bar{\sigma}_1(s', u_2), R), (U, \bar{\sigma}_1(s', u_2), D)\} \,.
\end{align*}
 
\noindent
Here, $\bar{\sigma}_1 \in \mc{P}$ is what we call the default script of the $\mathit{MFT}$. 
\end{example}

\begin{definition}[Move-Fixed Tree]
For a given state $s$, a subset of unrestricted units of $\mc{U}_i$ in $s$, a set of scripts $\mc{P}$, a default script $\bar{\sigma} \in \mc{P}$, and an action $a$ returned by the search algorithm's first step, a Move-Fixed Tree ($\mathit{MFT}$) is a tree rooted at $s$ with the following properties.
\begin{enumerate}
\item The set of abstracted legal actions for player $i$ at the root $s$ of the $\mathit{MFT}$ is limited to actions $a'$ that have moves $a'[u]$ fixed to $a[u]$, for all restricted units $u$; 
\item The set of abstracted legal actions for player $i$ at states $s'$ descendants of $s$ is limited to actions $a'$ that have moves $a'[u]$ fixed to $\bar{\sigma}(s', u)$, for all 
restricted units $u$;  
%where $\bar{\sigma} \in \mc{P}$ is the $\mathit{MFT}$'s default script.
\item The only abstracted legal action for player $-i$ at any state in the $\mathit{MFT}$ is defined by fixing the move returned by $\bar{\sigma}$ to all units in $\mc{U}_{-i}$. 
\end{enumerate}
\end{definition}

%In summary, 
By searching in the $\mathit{MFT}$, ABCD searches for moves for the unrestricted units while the moves of all other units, including the opponent's units, are fixed. We fix the opponent's moves to the NOKAV (our default script) as was done in previous work \cite{ChurchillB13,WangCLHT16,lelis2017}. 
%By fixing the opponent's moves to NOKAV we reduce the size of the game tree, thus allowing the search to reach deeper levels during planning. 
By fixing the opponent's moves to NOKAV we are computing a best response to NOKAV, and in theory, this could make our player exploitable. However, likely due to the real-time constraints, in practice one tends to derive more effective strategies by fixing the opponent to NOKAV, as mentioned in previous works \cite{lelis2017}. The development of action abstraction schemes different than using NOKAV for the opponent is an open research question. 
%Churchill and Buro \shortcite{ChurchillB13}, Wang et al. \cite{WangCLHT16}, and Lelis \shortcite{lelis2017} all fix the opponents moves
%
%The output of GAB is the action returned by ABCD's search in the $\mathit{MFT}$. %Note that, 
%due to the little time allowed for planning, PGS might not reach a local maximum before reaching GAB's overall time limit. It can also happen that 
%all action vectors considered by ABCD might have a lower $\Psi$-value 
%according to the evaluation function used 
%than the action $a$ returned by PGS. In that case 

Let $s_1$ and $s_2$ 
%be the action vectors 
be the states 
returned by the transition function $\mc{T}$ after applying action $a_1$ returned by GAB's first step (PGS) and action $a_2$ returned by GAB's second step (ABCD), respectively, from the state $s$ representing the game's current decision point. GAB returns 
%the action returned by PGS 
$a_1$ 
if $\Psi(s_1) > \Psi(s_2)$, and 
%the action returned by ABCD, 
$a_2$
otherwise. 
%Here, $\Psi'$ is the ``corrected'' evaluation function derived from GAB's ABCD search, which employs $\Psi$ as its evaluation function. ABCD performs an iterative-deepening depth-first search from state $s$ and it applies the original $\Psi$ function to all states $s'$ encountered at depth $d$ in the tree rooted at $s$. Here, $d$ is the depth bound of ABCD's last iteration. The $\Psi'$-values of $s_1$ and $s_2$ are computed by propagating up the tree the $\Psi$-values of the states $s'$ with the standard minimax rule. The ``corrected'' evaluation $\Psi'$ tends to be more accurate than $\Psi$ because it is computed from states deeper into the game tree. 
%from the propagated $\Psi$-values of the states $s'$.

%an iterative deepening  searches deep into the $\mathit{MFT}$
%returned by PGS and ABCD. 
%GAB returns $a$ if it leads to a state $s$

%returned by PGS if its ABCD search is unable to find an action vector with better $\Psi$ value.

\subsubsection{Stratified Alpha-Beta Search (SAB)}

The difference between SAB and GAB is the search algorithm used in their first step: while GAB uses PGS, SAB uses SSS. 
%In contrast with PGS, SSS apply your search in the space of script assignments induced by a type system. %We stop SSS's search in SAB's first step when we do not observe an improvement in the evaluation value during the process of search for better scripts to type system. We refer the reader to \citeauthor{lelis2017} \shortcite{lelis2017} for more details.
The second step of SAB follows exactly the second step of GAB. % described above. 

%is guaranteed to return an action with $\Psi$-value at least as good as the action returned by SSS. 

\subsubsection{GAB and SAB for Uniform Abstractions}

%Note that, 
For any state $s$, the value of $\Psi(\mc{T}(s, a_i, a_{-i}))$ for the action $a_i$ returned by PGS is a lower bound for the $\Psi$-value of the action returned by GAB. 
%
%GAB is guaranteed to return an action with $\Psi$-value not lower than the $\Psi$-value of the action returned by PGS. 
Similarly, SAB has the same guarantee over SSS. This is because the second step of GAB and SAB are performed only after a local maximum is reached. If the second step is unable to find an action with larger $\Psi$ than the first step, both GAB and SAB return the action encountered in the first step. %Thus, if the $\Psi$ function is informative enough, one should expect GAB and SAB to outperform PGS and SSS in practice. 
%
%Due to this advantage of GAB over PGS and SAB over SSS, empirical improvements of GAB and SAB over their counterparts could be due to the search algorithm and not to the abstraction scheme. 
We introduce variants of GAB and SAB called GAB$_{\mc{P}}$ and SAB$_{\mc{P}}$ that search in uniformly abstracted spaces to compare asymmetric with uniform abstractions. 

The difference between GAB and SAB and their variants GAB$_{\mc{P}}$ and SAB$_{\mc{P}}$ is that the latter 
only account for unit moves in $\mc{M}(s, u, \mc{P})$ for all  $s$ and $u$  
in their ABCD search. 
That is, in their second step search, GAB$_{\mc{P}}$ and SAB$_{\mc{P}}$ only consider actions $a'$ for which the moves $a'[u]$ for restricted units $u$ are fixed (as in GAB's and SAB's $\mathit{MFT}$) and the moves $a'[u']$ for unrestricted units $u'$ that are in $\mc{M}(s, u', \mc{P})$. 
%In summary, GAB$_{\mc{P}}$ and SAB$_{\mc{P}}$ search in uniformly abstracted spaces while GAB and SAB search in asymmetrically abstracted spaces.  By comparing GAB to GAB$_{\mc{P}}$ and SAB to SAB$_{\mc{P}}$ we are comparing the effectiveness of asymmetric abstractions over uniform ones.  

%In some sense, 
GAB$_{\mc{P}}$ and SAB$_{\mc{P}}$ focus their search on a subset of units $\mc{U}'$ by searching deeper into the game tree with ABCD for $\mc{U}'$. 
In addition to searching deeper with ABCD, GAB and SAB focus their search on a subset of units $\mc{U}'$ by accounting for all legal moves of units in $\mc{U}'$ during search. 
%as ABCD tends to search much deeper into the game tree than PGS and SSS. % while defining the moves of a subset of units. 
If granted enough computation time, optimal algorithms using $\Omega$ derive stronger strategies than optimal algorithms using $\Phi$  
%GAB and SAB would find stronger strategies than GAB$_{\mc{P}}$ and SAB$_{\mc{P}}$ 
(Theorem \ref{theorem:abs}). 
In practice, due to the real-time constraints, algorithms are unable to compute optimal strategies for most of the decision points. We analyze empirically, by comparing GAB$_{\mc{P}}$ to GAB and SAB$_{\mc{P}}$ to SAB, which abstraction scheme allows one to derive stronger strategies. % while focusing the search on a subset of units. 

%The difference between GAB$_{\mc{P}}$ and SAB$_{\mc{P}}$ and algorithms that use asymmetric abstractions 

\subsection{Strategies for Selecting Unrestricted Units}

%Asymmetric abstractions depends on (i) the unrestricted set size, and on (ii) the units composing the unrestricted set. 
In this section we describe three strategies for selecting the unrestricted units. A selection strategy   
%We introduce three strategies for selecting $N$ units to compose the unrestricted set at a given state $s$ representing a decision point. 
%The strategies for selecting unrestricted units 
receives a state $s$ and a set size $N$ and returns a subset of size $N$ of player $i$'s units. The selection of unrestricted units is dynamic as the strategies can choose different unrestricted units at different states. 
%Then, we test empirically these strategies with various set sizes. 
%While describing the strategies 
%We assume unrestricted sets of size $N$ 
%Let $N$ be the unrestricted set size. 
Ties are broken randomly in our strategies.  

\begin{enumerate}
\item \textbf{More attack value (AV+)}. Let 
%$av$-value of a unit be defined as
$av(u) = \frac{dpf(u)}{hp(u)}$. AV+ selects the $N$ units with the largest $av$-values, which allows search algorithms to  
%By selecting the units with the largest $av$-values, % to compose the unrestricted set, 
%AV+ 
%is similar to HP- as they both 
provide finer control to units with low $hp$-value
%, with the difference being that the former selects units with low $hp$-value 
and/or large $dpf$-value.
%
%The amount of damage $u$ can cause per frame of the game is defined as $dpf(u) = \frac{d(u)}{cd(u) + 1}$ (we add one to the denominator to ensure a valid operation, even if $cd = 0$). Finally, the attack value of a unit is defined as $av(u) = \frac{dpf(u)}{hp(u)}$. A unit $u$ has a large attack value if $u$ 
%is able to cause lots of damage per frame 
%and/or if it has very few hit points. 
%
%
%, while the latter only accounts for $hp$. 
This strategy is expected to perform well as it might be able to preserve longer in the match the units which are about the be eliminated from the match and have good attack power. 
%with low $hp$ and large $dpf$.
%are able to cause more damage. 
\item \textbf{Less attack value (AV-)}. AV- selects the $N$ units with the smallest $av$-values. We expect this strategy to be outperformed by AV+, 
%as we believe to be better to provide a finer control to units with smaller $hp$ and larger $dpf$, instead of the opposite, 
as explained above.  
\item \textbf{Random (R)}. R randomly selects $N$ units at $s_{init}$ to be the unrestricted units. R replaces an unrestricted unit that has its $hp$-value reduced to zero by randomly selecting a restricted unit. This is a domain-independent strategy that in principle could be applied to other multi-unit domains. % in which one controls multiple units in an adversarial scenario such as soccer for robots. 
\end{enumerate}

\section{Empirical Methodology}

We use SparCraft\footnote{github.com/davechurchill/ualbertabot/tree/master/SparCraft} as our testbed, which is a simulation environment of Blizzard's StarCraft.   
In SparCraft the unit properties such as hit points are exactly the same as StarCraft. However, SparCraft does not implement fog of war, collisions, and unit acceleration \cite{ChurchillB13}. We use SparCraft because it offers an efficient forward model of the game, which is required by search-based methods. All experiments are run on 2.66 GHz CPUs.

\subsection{Combat Configurations}

We experiment with units with different $hp$, $d$, and $r$-values. We use $\uparrow$ to denote large and $\downarrow$ to denote small values. %$hp$ and $d$-values. 
Also, we call $u$ a melee unit if $u$'s attack range equals zero ($r = 0$), and we call $u$ a ranged unit if $u$ is able to attack from far ($r > 0$). 
Namely, we use the following unit kinds: Zealots (Zl, $\uparrow$$hp$, $\uparrow$$d$, melee), Dragoons (Dg, $\uparrow$$hp$, $\uparrow$$d$, ranged), Zerglings (Lg, $\downarrow$$hp$, $\downarrow$$d$, melee), Marines (Mr, $\downarrow$$hp$, $\downarrow$$d$, ranged). 

We consider the combat scenarios where each player controls units of the following kinds: (i) Zl; (ii) Dg; (iii) Zl and Dg; (iv) Zl, Dg, and Lg; and (v) Zl, Dg, Lg, and Mr. 
We experiment with matches with as few as 6 units 
%on each side 
and as many as 56 units on each side. 
The largest number of units controlled by a player in a typical StarCraft combat is around 50 \cite{ChurchillB13}. The first two columns of Table \ref{tab:final_results} show the 20 combat configurations used in the experiments. The number of units is distributed equally amongst all kinds of units. For example, the scenario Zl+Dg+Lg+Mr with a total number of 56 units has 14 units of each kind. 

%We experiment with start states containing Zealots (Zl), Dragoon (Dg), Zerglings (Ling) and Marines (Mr) units. Zealots have lower $hp$-value than Dragoons, and Dragoons have a much larger attack range than Zealots. Zerglings and Marines have lower $hp$-value. Zerglings have a small attack range and Marines have a medium attack range. We consider combat scenarios composed of only Zealots and of only Dragoons with the following number of units controlled by each player: 8, 16, 32, and 50. We also consider combat configurations containing a mixture of Zealots, Dragoons, Zerglings and Marines with 8, 16, 18, 32, 42, 50, 54 and 56 number of units controlled by each player, presented in the format ($n$), and the quantity of type's units used in the test, presented in the format $t$. For example, in a specific configuration all players will control Zealots, Dragoons and Zerglings totality 42 units, meaning that each player controls $\frac{n}{t}$ Zealots, $\frac{n}{t}$ Dragoons and $\frac{n}{t}$ Zerglings: (14,14,14). In total we consider 20 distinct combat configurations: 4 with only Zealots, 4 with only Dragoons, 4 with a mixture of Zealots and Dragoons, 4 with a mixture of Zerglings, Zealots and Dragons and 4 with a misture of all type's unit.
%We note that 50 is likely the largest number of units controlled by a player in a typical StarCraft match \cite{ChurchillB13}.

The units are placed in a walled arena with no obstacles of size 1280 $\times$ 780 pixels; the largest unit (Dragoon) is approximately 40 $\times$ 50 pixels large. The walls ensure finite matches by preventing units from indefinitely moving away from the enemy. 
%For each combat scenario we generate 
%1,000 
%We generate 
%start states as explained by Lelis \shortcite{lelis2017}. %Churchill and Buro \shortcite{ChurchillB13}. 
%Namely, 
Player $i$'s units are placed at a random coordinate to the right of the center of the arena (with distance varying from $0$ and $128$ pixels). %, which we define to be the coordinate $(0, 0)$. 
Player $-i$'s units are placed at a symmetric position to the left of 
the center. %$(0, 0)$. 
%The coordinates are chosen so that the units are at a distance in between $0$ and $128$ pixels from the center. %$(0, 0)$. 
Then, 
%to ensure that no unit starts within the attack range of an enemy unit, 
we add $220$ pixels to the $x$-coordinate of player $i$'s units, and subtract $220$ pixels from the $x$-coordinate player $-i$'s units, thus increasing the distance between enemy units by $440$ pixels. 
We use $\mc{P} = \{$NOKAV, Kiter$\}$ and %and $\mc{P}_2 = \{$NOKAV, Kiter, Cluster$\}$
 a time limit of 40 milliseconds for planning in all experiments. 
 
%PGS, GAB, and GAB$_{\mc{P}}$ use the $\Psi$ function described by Churchill and Buro \shortcite{ChurchillB13}, and SSS+, SAB, and SAB$_{\mc{P}}$ use 
%the one 
%the $\Psi$ function 
We use the $\Psi$ function described by Churchill \etal\ \shortcite{ChurchillSB12}. %, which uses %Both $\Psi$ functions use the 
%LTD2. % introduced by Kovarsky and Buro \shortcite{Kovarsky2005}. %LTD2 accounts for both players' units $hp$ and $dpf$. 
%LTD2 evaluates a state $s$ with $\mc{U} = \mc{U}_i \cup \mc{U}_{-i}$ as follows. 
%
%\begin{equation*}
%LTD2(\mc{U}_i, \mc{U}_{-i}) = 
%\sum_{u \in \mc{U}_i} \sqrt{hp(u)} \cdot dpf(u) - \sum_{u \in \mc{U}_{-i}} \sqrt{hp(u)} \cdot dpf(u) \,.
%\end{equation*}
%
%\noindent
%However, 
Instead of evaluating state $s$ directly with LTD2, 
%Churchill \etal \shortcite{ChurchillSB12} showed empirically that 
our $\Psi$ 
simulates the game forward from $s$ for 100 state transition steps  until reaching a state $s'$; we then use the LTD2-value of $s'$ as the $\Psi$-value of $s$. The game is simulated from $s$ according to the NOKAV script for both players. 

\subsection{Testing Selection Strategies and Values of $N$}

\begin{table}[t]
\centering
\begin{small}
\begin{tabular}{c c c c c c c}
\hline
\multicolumn{7}{c}{GAB vs. PGS}                                                                                                                                                                                                                                                                                                                                                                                                                                                                                                                                   \\
\hline
%\multicolumn{7}{c}{GAB vs. PGS}                                                                                                                                                                                                                                                                                                                                                                                                                                                                                                                                   \\
\multirow{2}{*}{Strategy}& \multicolumn{5}{c}{Unrestricted Set Size $N$} & \multirow{2}{*}{Avg.} \\
\cline{2-6}
& 2 & 4 & 6 & 8 & 10 &                       \\ \hline \hline
AV+ & 0.88  &  0.92 & 0.89 & 0.87 & 0.86       & 0.88      \\
AV- & 0.69  & 0.76       & 0.78 &  0.82 &  0.82 & 0.77       \\
R   & 0.78  & 0.86       & 0.87 &  0.88   & 0.88 & 0.85                          
\\
\hline
\hline
\multicolumn{7}{c}{SAB vs. SSS}                                                                                                                                                                                                                                                                                                                                                                                                                                                                                                                                   \\
\hline
\multirow{2}{*}{Strategy}& \multicolumn{5}{c}{Unrestricted Set Size $N$} & \multirow{2}{*}{Avg.} \\
\cline{2-6}
& 2 & 4 & 6 & 8 & 10 &                       \\ \hline \hline
AV+ & 0.89  & 0.92 & 0.90 & 0.88       & 0.90  & 0.87      \\
AV- & 0.69  & 0.76       & 0.78 & 0.70       & 0.82 & 0.75  \\
R   & 0.75  & 0.80       & 0.83 & 0.84 & 0.85 & 0.81                                                                                       
\\
\hline
\end{tabular}
\caption{Winning rate of GAB against PGS and of SAB against SSS for different strategies and set sizes. %for selecting the unrestricted units and different unrestricted unit set sizes.
}
\label{tab:avg_behavior}
\end{small}
\end{table}

%\begin{table}[t]
%\centering
%\begin{small}
%%\setlength{\tabcolsep}{2pt}
%\begin{tabular}{c c c c c c c}
%\hline
%\multicolumn{7}{c}{SAB vs. SSS}                                                                                                                                                                                                                                                                                                                                                                                                                                                                                                                                   \\
%\hline
%\multirow{2}{*}{Strategy}& \multicolumn{5}{c}{Unrestricted Set Size} & \multirow{2}{*}{Avg.} \\
%\cline{2-6}
%& 2 & 4 & 6 & 8 & 10 &                       \\ \hline \hline
%AV+ & 0,89  & 0,92 & 0,90 & 0,88       & 0,90  & 0.87      \\
%AV- & 0,69  & 0,76       & 0,78 & 0,70       & 0,82 & 0.75  \\
%R   & 0,75  & 0,80       & 0,83 & 0,84 & 0,85 & 0.81                                                                                       
%\\
%\hline
%\end{tabular}
%\caption{Winning rate of SAB against SSS for different strategies and set sizes.}
%\label{tab:SAB_avg_behavior}
%\end{small}
%\end{table}

First, we test different strategies for selecting unrestricted units 
as well as 
different values of $N$. 
%the size of the unrestricted set. 
We test GAB against PGS and SAB against SSS 
(the algorithms used in the first step of GAB and SAB) 
with AV+, AV-, and R, with $N$ varying from 1 to 10. Table \ref{tab:avg_behavior} shows the average winning rates of GAB  and SAB in 100 matches for each of the 20 combat configurations. 
Since the winning rate does not vary much with $N$, we show the winning rate of only even values of $N$. 
%The highlighted cells show the set sizes that perform best for each strategy. 
The ``Avg.'' column shows the average across 
%different set sizes, for 
all $N$ (1 to 10). 

Both GAB and SAB outperform 
%the player following a strategy derived by PGS 
their base algorithms  
for all selection strategies and $N$ values tested, even with the domain-independent R. 
The strategy that performs best is AV+, which %As expected, 
%The best overall result was obtained by 
%AV+ 
obtains a winning rate of 0.92 
with $N$ of 4 for both GAB and SAB. The winning rate can 
vary considerably depending on the selection strategy for a fixed $N$. 
%be much lower depending on the strategy and unrestricted set size. 
For example, for $N$ of 2, PGS and SAB with AV+ obtain a winning rate of 0.88 and 0.89, respectively, while they obtain a winning rate of only 0.69 with AV-. % for both GAB and SAB. 
%AV+ provides finer control to units with lower hit points, which may allow them to survive longer. % in the match.  
These results demonstrate the importance of 
%dividing the search effort unevenly amongst units so that the resulting strategy is focused on a core subset of units. 
carefully selecting the set of units for which the algorithm will focus its search on. 
%, a core feature of asymmetric abstractions  not present in uniform abstractions.   

%The best overall result was obtained by AV+ with unrestricted set of size 4: winning rate of 0.92. 
%The results presented in Table \ref{tab:avg_behavior} show that the unrestricted set size affects the results, but the impact of the set size is modest for the range of values tested. For example, the worst result obtained for AV+ is 0.86, while its best is 0.92. This happens because sets with similar size will often result in the same asymmetric abstraction. For example, in combat scenarios in which the player controls 8 units, set sizes of 8, 9, and 10 result in the same asymmetric abstraction as all units are in the unrestricted set. Also, GAB and SAB run ABCD only if the first step finishes by encountering a local maximum. Thus, in some decision points there will not be enough time to run ABCD as the first step will use all the time available for planning. In these cases the action GAB and SAB return will be independent of the set size.

%The results of SAB against SSS+ are similar and showed in Table \ref{tab:SAB_avg_behavior}. That is, SAB exploits SSS+ with all strategies and unrestricted set sizes tested, AV+ is the best performing strategy, and the winning rate does not vary much while varying the set size. SAB using AV+ and set size of 4 has a winning rate of 0.92 against SSS+. 

Although GAB$_\mc{P}$ and SAB$_\mc{P}$ do not search in asymmetrically abstracted spaces, their performance also depends on the set of units controlled in the algorithms' ABCD search. Thus, we tested GAB$_\mc{P}$ and SAB$_\mc{P}$ with AV+, AV-, and R for selecting the units to be controlled in the algorithms' ABCD search. We also tested different number of units controlled in such searches: we tested set sizes from 1 to 10. Similar to the GAB and SAB experiments, we tested GAB$_\mc{P}$ against PGS and SAB$_\mc{P}$ against SSS; the detailed results are also omitted for space. The highest winning rate obtained by GAB$_\mc{P}$ against PGS was 0.74 while using the AV+ strategy to control 9 units in its ABCD search. The highest winning rate obtained by SAB$_\mc{P}$ against SSS+ was 0.78 while using the R strategy to control 9 units in its ABCD search. 

%These results suggest that 
GAB and SAB tend to perform best while controlling a smaller set of units (4 units in our experiment) in their ABCD search than GAB$_{\mc{P}}$ and SAB$_{\mc{P}}$ (9 units). This is because GAB and SAB's ABCD search does not restrict the moves of the units, while GAB$_{\mc{P}}$ and SAB$_{\mc{P}}$'s ABCD search does. GAB$_{\mc{P}}$ and SAB$_{\mc{P}}$ are able to effectively search deeper for a larger set of units than GAB and SAB. On the other hand, GAB and SAB are able to encounter finer strategies to the unrestricted units. Next, we directly compare these approaches with a detailed empirical study. 
%searches in a richer space as the moves of the units are not restricted. 

\subsection{Asymmetric versus Uniform Abstractions}

%Next, we test algorithms that use similar search strategies and  differ on the kind of action abstraction used. Namely, 
%Next, 
We test  GAB, GAB$_{\mc{P}}$ and PGS (G-Experiment); and SAB, SAB$_{\mc{P}}$, SSS (S-Experiment). 
GAB, GAB$_{\mc{P}}$, SAB, and SAB$_{\mc{P}}$ use the best performing unrestricted set size $N$ and selection strategies 
%for defining the units controlled in their ABCD searches 
as described above. %in the previous section. 

\begin{table}[t]
\centering
%\scriptsize 
\setlength{\tabcolsep}{3.8 pt}
\begin{tabular}{p{0.1cm} p{0.2cm} c|c c c|c c c} 
\hline
%\multicolumn{9}{c}{\textbf{Tests with set of scripts $\mc{P}_1$ }} \\
%\hline
 \multicolumn{ 3 }{c|}{\multirow{2}{*}{\#Units}}	 & 	 \footnotesize{GAB$_{\mc{P}}$} 	 & 	 \multicolumn{ 1 }{c}{\footnotesize{GAB}}  	 &  \footnotesize{GAB} & 	 \footnotesize{SAB$_{\mc{P}}$} 	 & 	 \multicolumn{ 1 }{c}{\footnotesize{SAB}}  	 &  \footnotesize{SAB} \\
& & & \footnotesize{PGS} & \footnotesize{PGS} & \footnotesize{GAB$_{\mc{P}}$} & \footnotesize{SSS} & \footnotesize{SSS} & \footnotesize{SAB$_{\mc{P}}$} \\
\hline
\hline
\multicolumn{ 2 }{r}{\multirow{4}{*}{\rotatebox[origin=c]{90}{ \textbf{Zl} }}} 
   & \textbf{(8)} 	 &  0.73 	 & 	0.72 	 & 	0.52 & \cg\	0.65 	 & \cg\	0.95  & \cg\ 0.93 \\
& & \textbf{(16)} 	 & 	0.78 	 & 	0.79 	 & 	0.57 & \cg\ 0.70 	 & \cg\	0.96  & \cg\ 0.94 \\
& & \textbf{(32)} 	 & 	0.77 	 & 	0.81 	 &  0.54 & \cg\	0.72 	 & \cg\	0.93  & \cg\ 0.81 \\
& & \textbf{(50)} 	 & 	0.80 	 & 	0.78 	 & 	0.50 & \cg\	0.69 	 & \cg\	0.90  & \cg\ 0.76 \\
\hline
\hline
\multicolumn{ 2 }{r}{\multirow{4}{*}{\rotatebox[origin=c]{90}{\textbf{Dg}}}} 
   & \textbf{(8)} 	 & 	0.69 	 & 	0.94 	 & 	0.88 & \cg\	0.60 	 & \cg\	0.91 	 & \cg\ 0.88 \\
&  & \textbf{(16)} 	 & 	0.71 	 & 	0.85 	 &  0.84 & \cg\	0.62 	 & \cg\	0.93 	 & \cg\ 0.88 \\
&  & \textbf{(32)} 	 & 	0.68     & 	0.81 	 & 	0.82 & \cg\	0.65 	 & \cg\	0.88 	 & \cg\ 0.81 \\
&  & \textbf{(50)} 	 & 	0.64 	 & 	0.78 	 & 	0.78 & \cg\	0.67 	 & \cg\	0.87 	 & \cg\ 0.79 \\
\hline
\hline
\multicolumn{ 2 }{r}{\multirow{4}{*}{\rotatebox[origin=c]{90}{\textbf{Zl+Dg}}}} %& \multirow{4}{*}{\rotatebox[origin=c]{90}{}} 
   & \textbf{(8)}   & 	0.64 	 & 	0.76 	 & 	0.68 & \cg\	0.59 	 & \cg\	0.93 	 & \cg\	0.90 \\
&  & \textbf{(16)}   & 	0.66 	 & 	0.82	     & 	0.78 & \cg\	0.66 	 & \cg\	0.93     & \cg\	0.86 \\
&  & \textbf{(32)} & 	0.66     & 	0.79 	 & 	0.79 & \cg\	0.64     & \cg\	0.91 	 & \cg\	0.81 \\
&  & \textbf{(50)} & 	0.65 	 & 	0.74 	 & 	0.71 & \cg\	0.63 	 & \cg\	0.90 	 & \cg\	0.77 \\
\hline
\hline
\multirow{4}{*}{\rotatebox[origin=c]{90}{\textbf{Zl}+\textbf{Dg}}} & \multirow{4}{*}{\rotatebox[origin=c]{90}{\textbf{+Lg}}} 
   & \textbf{(6)}    & 	0.58 	 & 	0.94	     & 	0.91 & \cg\	0.59 	 & \cg\	0.94 	 & \cg\ 0.94 \\
&  & \textbf{(18)}    & 	0.66 	 & 	0.93 	 & 	0.90 & \cg\	0.67 	 & \cg\	0.94     & \cg\ 0.89  \\
&  & \textbf{(42)} & 	0.66     & 	0.89 	 & 	0.89 & \cg\	0.65     & \cg\	0.92 	 & \cg\ 0.83 \\
&  & \textbf{(54)} & 	0.64 	 & 	0.86 	 & 	0.89 & \cg\	0.63 	 & \cg\	0.89 	 & \cg\ 0.79 \\
\hline
\hline
\multirow{4}{*}{\rotatebox[origin=c]{90}{\textbf{Zl}+\textbf{Dg}}} & \multirow{4}{*}{\rotatebox[origin=c]{90}{\textbf{Lg}+\textbf{Mr}}} 
   & \textbf{(8)}   & 	0.60 	 & 	0.92 	 & 	0.88 & \cg\	0.58 	 & \cg\	0.95 	 & \cg\	0.94 \\
&  & \textbf{(16)}   & 	0.64 	 & 	0.94 	 &  0.91 & \cg\	0.59 	 & \cg\	0.95	     & \cg\	0.91 \\
&  & \textbf{(40)} & 	0.65     & 	0.92 	 & 	0.90 & \cg\	0.61	     & \cg\	0.91 	 & \cg\	0.82 \\
&  & \textbf{(56)} & 	0.66 	 & 	0.92 	 & 	0.90 & \cg\	0.60 	 & \cg\	0.85 	 & \cg\	0.75 \\
\hline
\end{tabular}
\caption{Top player's winning rate against bottom player. %Column \# Units shows the total of units in each match, however this total is divided equally for each unit type. 
}
\label{tab:final_results}
\end{table}

The winning rates in 1,000 matches of the algorithms in the G-Experiment are shown on the lefthand side of Table \ref{tab:final_results}. 
The first two columns of the table specify the kind and the total number of units controlled by each player. % in the match, respectively. 
%The first column of the table specifies the kind of units used in the matches. Column ``\# Units'' specifies the number of units controlled by each player, 
%If the player controls more than one kind of unit, the number of units is divided equally by the number of kinds of units controlled. 
%For example, if we look for line where   Zealots (Zl) and Dragoons(Dg) are showed, the number of units is 8, this mean 4 units Zealots and 4 units Dragoons existing in this scenario. 
The remaining columns show the winning rate of the top algorithm, shown in the first row of the table, against the bottom algorithm. For example, in matches with 16 Zealots and 16 Dragoons (total of 32 units) GAB defeats PGS in 79\% of the matches.  The winning rates of the algorithms in the S-Experiment are shown on the righthand side of the table.

%We will analyze the G-Experiment results first. %The table shows the combat configurations in which GAB$_{\mc{P}}$ and GAB outperform PGS. 
We observe in the third and fourth columns of the table that both GAB$_{\mc{P}}$ and GAB outperform PGS in all configurations tested. However, these results do not allow us to verify the effectiveness of asymmetric abstractions if analyzed individually. This is because both GAB$_{\mc{P}}$ and PGS search in uniformly abstracted spaces, and GAB's advantage over PGS could be due to the use of a different search strategy, and not due to the use of a different abstraction scheme. 
By comparing the numbers across the two columns we observe that GAB, which uses asymmetric abstractions, obtains substantially larger winning rates over PGS than GAB$_{\mc{P}}$, which uses uniform abstractions. For example, in matches with 8 Zealots and 8 Dragoons (16 units in total), GAB$_{\mc{P}}$'s winning rate is 0.66 against PGS, while GAB's is 0.82. 

The column GAB vs GAB$_{\mc{P}}$ of the table allows a direct comparison between uniform and  
asymmetric abstractions. 
GAB substantially outperforms GAB$_{\mc{P}}$ in almost all configurations, %except the Zealots-only configuration, but even there the 
and its winning rate 
%of GAB over GAB$_{\mc{P}}$ 
is never below 0.50. These results highlight the importance of 
focusing the search effort on a subset of units through an asymmetric abstraction. 
%dividing the search effort unevenly amongst units with an asymmetric abstraction as opposed to a uniform abstraction. % so that the resulting strategy is focused on a core subset of units.

%We conjecture that for Zealots-only combats the actions provided by the scripts NOKAV and Kiter already encode most of the information needed for deriving strong gameplay strategies. 
%GAB substantially outperforms GAB$_{\mc{P}}$ in the other  configurations, often with winning rates around 0.80. Similar results are observed for the S-Experiment 2 (see Table \ref{tab:final_results}). SAB tends to have larger winning rates over SSS+ than SAB$_{\mc{P}}$ where the set of scripts is $\mc{P}_1$. 
%The column SAB vs SAB$_{\mc{P}}$ of Table \ref{tab:final_results} shows that SAB is able to outperform SAB$_{\mc{P}}$ in all configurations tested where the set of scripts is $\mc{P}_1$. 

The results for the S-Experiment are similar to those of the G-Experiment: SAB has a higher winning rate over SSS than SAB$_{\mc{P}}$ and SAB substantially outperforms SAB$_{\mc{P}}$. 
%These results suggest that asymmetric abstractions allow one to derive stronger strategies than uniform abstractions. 

SAB's winning rate over SSS is often larger than GAB's over PGS. For example, in combat scenarios with Zealots only (Zl), GAB's largest winning rate over PGS is 0.81 (with 32 units), which is smaller than the smallest winning of SAB over SSS (0.90 with 50 units). This is likely because SAB's first step (SSS) tends to finish much more quickly than GAB's (PGS). SSS searches for actions for types of units, while PGS searches for actions for units directly, and the number of types tend to be much smaller than the number of units \cite{lelis2017}. As a result, SAB performs its second step more often than GAB, which allows SAB to derive finer strategies to its unrestricted units in more decision points than GAB. 
%Namely,
%the average number of times GAB executes its second step in a match was 72.8\% 
%By contrast, the average number of times SAB executes its second step in a match was 97.4\% 
%(the maximum was 100\% and the minimum was 86.6\%). 
In addition to executing the second step more often, SAB usually allows more computation time for its second step. SAB allowed 32.6 milliseconds of computation time on average for its second step, while GAB allowed 21.8 milliseconds on average for its second step.    

\subsubsection{Comparison of GAS with GAB and PGS}

We also ran experiments comparing GAS with GAB and PGS in  combat scenarios containing (i) Zl, (ii) Dg, and (iii) Zl and Dg; we used the same number of units shown in Table \ref{tab:final_results} for these scenarios. For each combat scenario we ran 1,000 matches. GAS won 55\% of the matches against PGS and only 14\% against GAB. These results highlight the significance of combining novel search algorithms with asymmetric abstractions. GAS is able to only marginally outperform PGS. By contrast, the two-step scheme used with GAB substantially outperforms both PGS and GAS. 

%GAB has two advantages over GAS. First, GAB is guaranteed to return an action that has $\Psi$-value no worse than the $\Psi$-value of the action returned by PGS. By contrast, GAS does not have such a guarantee. Depending on the decision point, the number of actions available to GAS might be much larger than one is able to evaluate within the time limit. In this situation, GAS might return an action that has a $\Psi$-value worse than the $\Psi$-value of the action returned by PGS as GAS might not evaluate the same set of actions evaluated by PGS. Second, GAB's ABCD will search deeper than GAS in the game tree, which may allow GAB to make more informed decisions than GAS, as the latter evaluates the actions directly from the state representing the current decision point, without performing any lookahead search. %, that does not deepen its search. 

%as the latter's uniform abstraction allows one to focus only on the actions considered as promising by the set of scripts. 

%In the S-Experiment we notice that the advantage of SAB over SAB$_\mc{P}$ reduces as one increases the number of units in the matches. This is because, often in the beginning of matches with more units, both SAB and SAB$_\mc{P}$ spend almost all the time available for planning with their SSS+ search; the ABCD search, when invoked, has almost no time to search. 

\section{Conclusions and Future Work}

%In this paper 
We introduced GAB and SAB, two search algorithms that use an abstraction scheme we call asymmetric action abstraction. For not being too restrictive while filtering actions and for 
%allowing search algorithms to 
%focus on 
%a particular 
assigning finer strategies to a particular 
subset of units, 
%algorithms using asymmetric abstractions 
GAB and SAB 
are able to substantially outperform the state-of-the-art search-based algorithms for RTS combats. 
%to derive stronger strategies than the strategies derived while searching in uniformly abstracted spaces. 
%We introduced GAB and SAB, algorithms that use asymmetric abstractions, and GAB$_{\mc{P}}$ and SAB$_{\mc{P}}$, algorithms that use uniform abstractions. % and a few strategies for creating asymmetric action abstractions. 
%Our empirical results showed that 
%search algorithms that use asymmetric abstractions that allow a finer control to units with fewer hit points tend to perform best. We also showed empirically that GAB and SAB are able to outperform not only state-of-the-art search algorithms, but also their counterparts, GAB$_{\mc{P}}$ and SAB$_{\mc{P}}$, which search in uniformly abstracted spaces. 
%
%Our algorithms that search in uniformly abstracted spaces, 
%We also introduced GAB$_{\mc{P}}$ and SAB$_{\mc{P}}$, two algorithms that search in uniformly abstracted spaces. GAB$_{\mc{P}}$ and SAB$_{\mc{P}}$ already represent an improvement over the state-of-the-art search-based algorithms for RTS combats. GAB and SAB represent yet a much larger improvement. % over the state of the art. 
%
As future work we intend to apply GAB and SAB to complete RTS games and to compare them to other search-based approaches designed to play complete games such as NaiveMCTS \cite{Ontanon13} and PuppetSearch \cite{barriga2017game}. 
We are also interested in developing algorithms that learn how to  select the unrestricted set of units in scenarios that appear in complete RTS games. 

\section{Appendix: Proofs}

The proof of Theorem \ref{theoremappendix:abs} hinges on the fact that one has access to more actions with $\Omega$ than with $\Phi$. This idea is formalized in Lemma \ref{lemma:subset}. 

\begin{lemma}
Let $\Phi$ be a uniform abstraction and $\Omega$ be an asymmetric abstraction, both defined with the same set of scripts $\mc{P}$. Also, let $\mc{A}'_i(s)$ be the set of actions available at state $s$ according to $\Phi$ and $\mc{A}''_i(s)$ the set of actions available at $s$ according to $\Omega$. 
$\mc{A}'_i(s) \subseteq \mc{A}''_i(s)$ for any $s$. % if both $\Phi$ and $\Omega$ are defined with the same set of scripts $\mc{P}$. 
\label{lemma:subset}
\end{lemma}
\begin{proof}
By definition, the actions in $\mc{A}'_i(s)$ are generated by the Cartesian product of $\mc{M}(s, u, \mc{P})$ for all $u$ in $\mc{U}_i$ in $s$. The actions in $\mc{A}'_i(s)$ are generated by the Cartesian product of $\mc{M}(s, u, \mc{P})$ for all $u$ in $\mc{U}_i \setminus \mc{U}'_i$ and of $\mc{M}(s, u)$ for all $u$ in $\mc{U}'_i$. Since, also by definition, $\mc{M}(s, u, \mc{P}) \subseteq \mc{M}(s, u)$, we have that $\mc{A}'_i(s) \subseteq \mc{A}''_i(s)$. 
\end{proof}

Let $\Sigma'_i$ and $\Sigma''_i$ be the set of player $i$'s strategies whose supports contain only actions in $\mc{A}'_i$ and $\mc{A}''_i$, respectively. Also, let $\Sigma_{-i}$ be the set of all player $-i$'s strategies. Lemma \ref{lemma:subset} allows us to write the following corollary. 

\begin{corollary}
 For abstractions $\Phi$ and $\Omega$ defined from the same set of scripts $\mc{P}$ we have that $\Sigma'_i \subseteq \Sigma''_i$. 
\label{corollary:subset}
\end{corollary}

%In the proofs that follow,  % for that use any action in $\mc{A}_{-i}$. 

%\begin{lemma}
%Let $\Phi$ be a uniform abstraction and $\Omega$ be an asymmetric abstraction, both defined with the same set of scripts $\mc{P}$.
%Consider a match with start state $s$ in which any single joint action from $s$ leads the match to a terminal state.  
%Let $V_i^{\Phi}(s)$ be the optimal value of the game computed by considering the space induced by $\Phi$; define $V_i^{\Omega}(s)$ analogously. We have that 
%$V_i^{\Omega}(s) \ge V_i^{\Phi}(s)$. 
%\label{lemma:abs}
%\end{lemma}
%\begin{proof}
%\begin{equation*}
%\begin{aligned}
%V_i^{\Omega}(s) =  \max_{\sigma_i \in \Sigma''_i} \min_{\sigma_{-i} \in \Sigma_{-i}} & \sum_{a_i \in \mc{A}_i(s)} \sum_{a_{-i} \in \mc{A}_{-i}(s)} \sigma_i(s, a_i) \cdot \\
%& \sigma_{-i}(s, a_{-i}) \cdot \mc{R}_i(\mc{T}(s, a_i, a_{-i}))
%\end{aligned}
%\end{equation*}
%\begin{equation*}
%\begin{aligned}
% \hspace{0.45in} \ge  \max_{\sigma_i \in \Sigma'_i} & \min_{\sigma_{-i} \in \Sigma_{-i}} \sum_{a_i \in \mc{A}_i(s)} \sum_{a_{-i} \in \mc{A}_{-i}(s)} \sigma_i(s, a_i) \cdot \\
%& \sigma_{-i}(s, a_{-i}) \cdot \mc{R}_i(\mc{T}(s, a_i, a_{-i})) = V_i^{\Phi}(s) \,.
%\end{aligned}
%\end{equation*}
%The first equality is the definition of the value of a zero-sum game. The inequality is because the max operation over a set $E$ cannot be smaller than the max over a subset of $E$, and $\Sigma'_i \subseteq \Sigma''_i$ (Corollary \ref{corollary:subset}). 
%The last equality is analogous to the first. 
%\end{proof}

\begin{theoremappendix}
Let $\Phi$ be a uniform abstraction and $\Omega$ be an asymmetric abstraction, both defined with the same set of scripts $\mc{P}$. For a finite match with start state $s$, let $V_i^{\Phi}(s)$ be the optimal value of the game computed by considering the space induced by $\Phi$; define $V_i^{\Omega}(s)$ analogously. We have that 
$V_i^{\Omega}(s) \ge V_i^{\Phi}(s)$.
\label{theoremappendix:abs}
\end{theoremappendix}
\begin{proof}
We prove the theorem by induction  
on the level of the game tree. The base case is given by leaf nodes $s_l$. Since $V_i^{\Omega}(s_l) = V_i^{\Phi}(s_l) = \mc{R}_i(s_l)$, the theorem holds. % for the base case. 
The inductive hypothesis is that $V_i^{\Omega}(s') \ge V_i^{\Phi}(s')$ for any state $s'$ at level $j+1$ of the tree. For any state $s$ at level $j$ we have that,
\begin{equation*}
\begin{aligned}
V_i^{\Omega}(s) =  \max_{\sigma_i \in \Sigma''_i} \min_{\sigma_{-i} \in \Sigma_{-i}} & \sum_{a_i \in \mc{A}(s)} \sum_{a_{-i} \in \mc{A}_{-i}(s)} \sigma_i(s, a_i) \cdot \\
& \sigma_{-i}(s, a_{-i}) \cdot V_i^{\Omega}(\mc{T}(s, a_i, a_{-i}))
\end{aligned}
\end{equation*}
\begin{equation*}
\begin{aligned}
 \hspace{0.45in} \ge  \max_{\sigma_i \in \Sigma'_i} & \min_{\sigma_{-i} \in \Sigma_{-i}} \sum_{a_i \in \mc{A}(s)} \sum_{a_{-i} \in \mc{A}_{-i}(s)} \sigma_i(s, a_i) \cdot \\
& \sigma_{-i}(s, a_{-i}) \cdot V_i^{\Phi}(\mc{T}(s, a_i, a_{-i})) = V_i^{\Phi}(s) \,.
\end{aligned}
\end{equation*}
%The first equality holds because $V_i^{\Omega}(s)$ can be computed by using the values of $V_i^{\Omega}(s')$, where $s'$ are the states at level $j+1$ in the tree. 
The first equality is the definition of the value of a zero-sum simultaneous move game. 
The inequality is because 
$\Sigma'_i \subseteq \Sigma''_i$ (Corollary \ref{corollary:subset}) and $V_i^{\Omega}(\mc{T}(s, a_i, a_{-i})) \ge V_i^{\Phi}(\mc{T}(s, a_i, a_{-i}))$, as $\mc{T}(s, a_i, a_{-i})$ returns a state at level $j+1$ of the tree (inductive hypothesis). The inequality also holds if the transition $\mc{T}(s, a_i, a_{-i})$ returns a terminal state $z$ at level $j+1$ as $V^{\Omega}_i(z) = V^{\Phi}_i(z) = \mc{R}_i(z)$. The last equality is analogous to the first one. 
\end{proof}

\section{Acknowledgements}

The authors gratefully thank FAPEMIG, CNPq, and CAPES for financial support, the anonymous reviewers for several great  suggestions, and Rob Holte for fruitful discussions and suggestions on an earlier draft of this paper.  

\bibliography{aaai}
\bibliographystyle{aaai}

\end{document}